\documentclass[preprint,12pt]{elsarticle}
\usepackage{nicefrac,amsmath,amssymb,inputenc,graphicx,xcolor,rotating,amsthm}
\usepackage{blindtext}
\usepackage{enumerate}
\usepackage{xcolor}
\usepackage{tcolorbox}
\tcbset{center, top=5pt, bottom=5pt, width = \textwidth,colback=pink!20!,colframe=pink!80!,coltitle=black}

\newtheorem{theorem}{Theorem}[section]
\newtheorem{corollary}[theorem]{Corollary}

\newtheorem{proposition}[theorem]{Proposition}

\newtheorem{problem}[theorem]{Problem}

\newcommand{\relu}{\text{ReLU}}
\newcommand{\I}{I_{\alpha}}
\newcommand{\Il}{I_{\lambda}}
\newcommand{\xphi}{\langle x,\phi_i\rangle}
\newcommand{\yphi}{\langle y,\phi_i\rangle}
\newcommand{\Ca}{C_{\alpha}}
\newcommand{\Sl}{S_{\lambda}}
\newcommand{\RR}{\mathbb{R}}
\newcommand{\NN}{\mathbb{N}}

\newcommand{\M}{\mathcal{M}}

\begin{document}

\begin{frontmatter}



\title{Optimal lower Lipschitz bounds for ReLU layers, saturation, and phase retrieval}

%% Author name
\author[slu]{Daniel Freeman}
\author[ari]{Daniel Haider}

\affiliation[slu]{organization={Department of Mathematics and Statistics, St. Louis University},%Department and Organization
            addressline={Ritter Hall 319}, 
            city={St. Louis},
            postcode={63103}, 
            state={MO},
            country={USA}}
\affiliation[ari]{organization={Acoustics Research Institute, Austrian Academy of Sciences},%Department and Organization
            addressline={Dominikanerbastei 16}, 
            city={Vienna},
            postcode={1010}, 
            country={AUT}}

%% Abstract
\begin{abstract}
    The injectivity of ReLU layers in neural networks, the recovery of vectors from clipped or saturated measurements, and (real) phase retrieval in $\RR^n$ allow for a similar problem formulation and characterization using frame theory.  In this paper, we revisit all three problems with a unified perspective and derive  lower Lipschitz bounds for ReLU layers and clipping which are analogous to the previously known result for phase retrieval and are optimal up to a constant factor.  
\end{abstract}

%     \includegraphics[width=\linewidth]{non-linearities.pdf}

% \begin{highlights}
% \end{highlights}

%% Keywords

\end{frontmatter}

\section{Introduction}
Many important functions in applications arise as a linear operator composed with  some simple non-linear function.
In engineering, the non-linear component often comes from unwanted technical constraints, such as a limited dynamic range in measurement devices that causes clipping or saturation effects.  For example, if the amplitude of an audio signal  exceeds the threshold of the recording equipment at any point then the device will often record the resulting waveform as having its top and bottom clipped at the threshold.  
In contrast to this, there are circumstances where non-linearities are intentionally incorporated into model designs to enhance their expressiveness. This is extensively made use of in neural networks, where linear maps are alternately concatenated with non-linear activation functions to build large and powerful layer-structured models. % \cite{goodfellow2016dlb}.
In both situations, it is fundamentally important to know  whether or not it is possible to recover the original input from the resulting non-linear measurements.  In the unintentional case, we want to know how to effectively design  measurement devices which are able to compensate for the resulting loss of information. In the intended case, we may want to balance the effect of the non-linearity and the information flow, i.e., get the desired (compression or sparsification) effect \textit{and} invertibility which has been leveraged to extend the usage of a model and enhance its interpretability~\cite{ardizzone2019inn, mahendran2015invert}.
In other words, we want to know under which conditions the associated \textit{non-linear measurement operator} is one-to-one on a certain domain of interest.

The maps that we consider here are composed of a linear operator from $\RR^n$ to $\RR^m$, followed by a non-linear function that is applied component-wise to the resulting coefficients. Combining engineering and machine learning terminology we will refer to the respective computational steps as \textit{measurement} and \textit{activation}. Let $(\phi_i)_{i=1}^m$ be a collection of measurement vectors in $\RR^n$ with $m\geq n$.  We say that $(\phi_i)_{i=1}^m$ is a frame for $\RR^n$ if the (linear) measurement map $x\mapsto\left(\langle x,\phi_i\rangle\right)_{i=1}^m$ is one-to-one from $\RR^n$ to $\RR^m$. Let further $\rho:\RR \rightarrow \RR$ be a non-linear activation function, and $X\subseteq \RR^n$ an input domain of choice. The resulting non-linear measurement operator is of the form
\begin{align}\label{eq:measure}
\begin{split}
    \Phi:X&\rightarrow \mathbb{R}^m\\
    x&\mapsto\left(\rho(\langle x,\phi_i\rangle)\right)_{i=1}^m.
\end{split}
\end{align}
Of course, studying the one-to-one property of $\Phi$ is only interesting if the non-linear function $\rho$ is not one-to-one itself. 
We focus on three  examples of activation functions where $\rho$ is not one-to-one, but for which it is possible for the non-linear measurement operator $\Phi$ to be one-to-one on $X$.\\

\noindent
\textbf{A. ReLU layers}.
%In neural networks, the ReLU function is used as an activation function that is applied coordinate-wise to the output of the affine linear mappings $x\mapsto (\xphi-\alpha_i)_{i\in I}$. 
%These maps are one of the most commonly used building blocks in neural networks.
Originally introduced to regularize the gradients in deep network architectures during training, these maps have established themselves as a powerful and easy-to-use ingredient for the design of neural networks~\cite{glorot2011relu,nair2010reluRBM}. A ReLU layer can be written as a non-linear measurement operator according to \eqref{eq:measure} when setting the activation function $\rho$ to be
\begin{align}\label{eq:reludef}
    \relu(t) = \max(0,t).
\end{align}
Usually, a bias value $\alpha_i\in \RR$ is assigned to every measurement vector $\phi_i$. The map then becomes $x\mapsto (\relu(\xphi-\alpha_i))_{i=1}^m$. The vector $\alpha\in \RR^m$ comprising the bias values is called the \textit{bias vector}, which acts as a threshold mechanism for the activation.
%When training a neural network, both, the entries of the weight vectors $\phi_i$ and bias values $\alpha_i$ are optimized to solve a specific given problem, such as classification or regression. 
As input domain, we consider $X=\RR^n$ \cite{bruna2014pooling,haider2024relu,puthawala2022relu}.\\

\begin{figure}
    \centering
    \includegraphics[width=\linewidth]{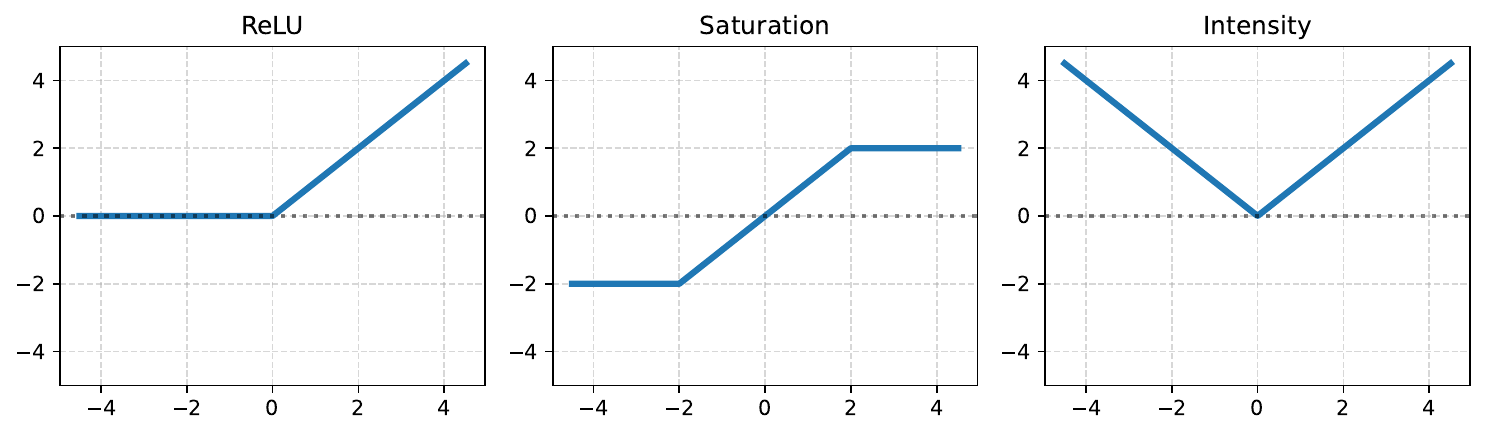}
    \caption{The plots show the non-linear activation functions used in A.-C.: The ReLU function, the saturation function with $\lambda = 2$, and the magnitude (intensity) function.}
    \label{fig:1}
\end{figure}

\noindent
\textbf{B. $\lambda$-Saturation or clipping}.
In this setting, it is assumed that measurements can only be recorded accurately up to a certain magnitude threshold level $\lambda>0$ and the maximum $\pm\lambda$ is returned whenever it is exceeded.  This is a common occurrence in signal processing and electrical engineering when using a measurement device with a limited dynamic range.
The corresponding activation function is the saturation function given by
\begin{align}\label{eq:satdef}
    \sigma_{\lambda}(t) = \operatorname{sign}(t)\cdot\min\left(\vert t\vert ,\lambda\right)
    % =\begin{cases}
    %     t &\textrm{ if } |t|\leq \lambda\\
    %     \operatorname{sign}(t)\cdot\lambda &\textrm{ otherwise.}
    % \end{cases}
\end{align}
Since all the measurements for vectors with very large norm would be saturated, we consider inputs only from a closed ball in $\RR^n$ \cite{alharbi2024sat, Foucart1,Foucart2,Laska}. Using a scaling argument, we may without loss of generality restrict the domain $X$ to be the unit ball $\mathbb{B}_{\RR^n}=\{x\in \RR^n: \Vert x\Vert \leq 1\}$.\\

\noindent
\textbf{C. Phase retrieval}.
In applications such as electron microscopy and coherent-diffraction imaging, researchers often work with devices which can only record the magnitude of the measurement values, and the phase information is lost. This map is sometimes called the intensity measurement operator \cite{balan2006phase}. Accordingly, the activation function is given by $t\mapsto |t|$.
Since, no matter what, we are only able to reconstruct the input up to a global sign, we consider the quotient space $\RR^n/_\sim$ as input domain, where the equivalence relation $\sim$ is given by $x\sim y$ if $x=\lambda y$ for some scalar $|\lambda|=1$.\\

For all three cases, the injectivity characterizations of the associated non-linear measurement operators have a similar frame theoretic  nature~\cite{alharbi2024sat,balan2006phase,bruna2014pooling,haider2024relu,puthawala2022relu}. %Yet, there are subtle differences which makes things easier or more complicated for the specific settings.
To ensure numerical stability of the recovery maps in applications, the non-linear measurement operators must not only be one-to-one but also be bi-Lipschitz embeddings. 
Recall that a map $f:X\rightarrow Y$ between metric spaces $(X,d_X)$ and $(Y,d_Y)$ is called {\em bi-Lipschitz} if  
%. That is, the mapping $\Phi$
%$F:\RR^n\rightarrow \RR^m$
%is bi-Lipschitz stable if
there are constants $\kappa_L,\kappa_U>0$ such that
\begin{equation}
    \kappa_L d_X(x,y) \leq d_Y(f(x),f(y)) \leq \kappa_U d_X(x,y)\hspace{.2cm}\textrm{ for all }x,y\in X.
\end{equation}
We call $\kappa_L$ a lower Lipschitz bound for $f$ and $\kappa_U$ an upper Lipschitz bound for $f$. A bound is called \textit{optimal} if the corresponding inequality is sharp.
%The optimal bi-Lipschitz constants for the (linear) measurement map correspond exactly to its smallest and largest singular values.
If the measurement vectors $(\phi_i)_{i=1}^m$ form a frame for $\RR^n$ then the linear map $x\mapsto (\langle x,\phi_i\rangle)_{i=1}^m$ is a bi-Lipschitz embedding from $\RR^n$ into $\RR^m$ and the optimal bi-Lipschitz constants correspond exactly to the linear map's smallest and largest singular values. The optimal \textit{frame bounds} of $(\phi_i)_{i=1}^m$ are the greatest positive value $A$ and least value $B$ so that $A\|x\|^2
\leq \sum_{i=1}^m |\langle x,\phi_i\rangle|^2\leq B\|x\|^2$ holds for all $x\in \RR^n$. Thus, $A$ and $B$ are the optimal frame bounds of $(\phi)_{i=1}^m$ if and only if
$\sqrt{A}$ and $\sqrt{B}$ are the optimal bi-Lipschitz bounds for the linear measurement map $x\mapsto (\langle x,\phi_i\rangle)_{i=1}^m$.  
We choose to use the terminology of frame bounds, although some papers on the topic express the corresponding inequalities in terms of singular values \cite{bandeira2014savingphase,puthawala2022relu}. 

All three of the activation functions that we consider are non-expansive mappings, i.e. they have an upper Lipschitz bound of $1$.  This simple property is very useful when working with non-linear operators, and it is directly exploited in \cite{CombettesWoodstock1} to solve a general class of inverse problems.   
For us, the fact that the activation functions are $1$-Lipschitz give that
if $B$ is the upper frame bound of $(\phi_i)_{i=1}^m$, then
the optimal upper Lipschitz bound for the non-linear measurement operator $\Phi$ always satisfies $\kappa_U\leq \sqrt{B}$. 
Obtaining lower Lipschitz bounds for $\Phi$ is much more difficult, and this will be the main focus of our paper. The following is our main result and provides the first dimension independent lower Lipschitz bounds of $\kappa_L$ for ReLU layers and $\lambda$-saturation.\\ 
% The optimal upper bound determines the stability of the (forward) map.
% In the context of machine learning, it is considered to be a good measure to evaluate robustness, generalization, and fairness of neural networks CITE.
% On the other hand, the optimal lower bound determines the stability of the inverse mapping, which can be interpreted as a measure of how reliable the reconstruction of an input will be in a numerical setting.

%So, our main results are summarized in the following theorem:\\

% The range of values that the optimal lower Lipschitz bound for phase retrieval can reach has been identified  in~\cite{bandeira2014savingphase}. For ReLU layers, there is only limited research on lower Lipschitz bounds so far and the current estimation for such a bound is far from optimal as it inversely scales with the number of measurement vectors~\cite{puthawala2022relu}.
% For saturation recovery the existence of a lower Lipschitz bound has been proven but an explicit estimation is still missing~\cite{alharbi2024sat}. Taking inspiration from phase retrieval, we are able to identify ranges of values that the optimal lower Lipschitz bounds for the other two problems can attend in an analogous manner.\\

\begin{tcolorbox}[title = {\textbf{Theorem 1 (Optimal lower Lipschitz bounds).}}]
\textit{Let $(\phi_i)_{i\in I}$ with $|I|=m$ be a frame for $\RR^n$, and let $\rho$ be the activation function for either a ReLU layer with bias $\alpha$ or $\lambda$-saturation.  Let $X\subseteq \RR^n$ and let $D\subseteq\RR$ be the closed interval where $\rho$ is one-to-one.  For each $x\in X$ let $I_\rho(x)=\{i\in I: \rho(\langle x,\phi_i\rangle)\in D\}$ and let $A(x)$ be the lower frame bound of $(\phi_i)_{i\in I_\rho(x)}$.
%, which is the square of the least singular value of the linear map $z\mapsto(\langle z,\phi_i\rangle)_{i\in I_\rho(x)}$.
We set $A_\rho=\min _{x\in X}A(x)$.
\vspace{1em}
\newline
(i) If $\rho$ is the activation function for an injective ReLU layer with bias $\alpha$ and $\kappa_L$ is the optimal lower Lipschitz bound for the map $x\mapsto\left(\rho(\langle x,\phi_i\rangle)\right)_{i\in I}$ from $X=\RR^n$ to $\RR^{m}$ then}
    \begin{equation}\label{E:relu}
        \tfrac{1}{2}\sqrt{A_\rho} \leq \kappa_L \leq \sqrt{A_\rho}.
    \end{equation}
\textit{(ii) If $\rho$ is the activation function for an injective saturated measurement operator with saturation level $\lambda>0$ and $\kappa_L$ is the optimal lower Lipschitz bound for the map $x\mapsto\left(\rho(\langle x,\phi_i\rangle)\right)_{i\in I}$ from $X=\mathbb{B}_{\RR^n}$ to $\RR^{m}$ then}
    \begin{equation}\label{E:sat}
        \min\left\{\tfrac{1}{2}\sqrt{A_\rho}, \lambda\right\} \leq \kappa_L \leq \sqrt{A_\rho}.
    \end{equation}
%    \textit{The numbers $A_\alpha, A_\lambda$ are related to the optimal lower Lipschitz bounds for those coordinates of the associated measurement operators that are not affected by the activation.}
\end{tcolorbox}

In the case of ReLU layers with $\alpha=\mathbf{0}$, bounds for $\kappa_L$ in terms of frame bounds of subcollections of $(\phi_i)_{i\in I}$ was first given in \cite{bruna2014pooling}.  However, the provided lower bound could be $0$ in some circumstances, even when the map was bi-Lipschitz.  This  was rectified  in~\cite{puthawala2022relu} where it was proven that $\sqrt{\tfrac{A_\rho}{2m}} \leq \kappa_L$, thus providing the first lower bound on $\kappa_L$ which is non-zero if and only if $\kappa_L$ is non-zero.  The necessity of $m$ in this inequality was investigated in \cite{RauschMasters}, where it was conjectured that one may replace $m$ with $2$.  
Our bound in \eqref{E:relu} verifies this conjecture and provides the first  dimension independent lower bound on $\kappa_L$ for ReLU layers. 

For $\lambda$-saturation, it was known that if $\lambda$ is greater than a critical threshold then the non-linear measurement operator is one-to-one if and only if it is bi-Lipschitz \cite{alharbi2024sat}.  However,  lower bounds for $\kappa_L$ were only known for certain random frames \cite{Foucart2,Foucart1,Laska}, and whether or not the saturated measurement operator is bi-Lipschitz at the critical threshold was listed as an open problem in \cite{alharbi2024sat}.  Our bound in \eqref{E:sat} gives the first explicit lower bound on $\kappa_L$, and consequently
%It is known that the non-linear measurement operator for $\lambda$-saturation is one-to-one on the unit ball if and only if $A_\rho>0$ \cite{alharbi2024sat}.
%Thus, our bounds in \eqref{E:sat}
%implying that the non-linear measurement operator for $\lambda$-saturation is one-to-one on the unit ball if and only it is bi-Lipschitz. This
solves the open problem  about the critical threshold.\\

The paper is structured as following. In Section \ref{sec:2}, we review the known injectivity characterizations for ReLU layers, saturation, and phase retrieval, and put them into direct relation as non-linear measurement operators.
%A focus here lies on the different sets of coordinates where the measurement vectors are affected by the non-linear activation functions in different ways.
This sets the stage for studying other settings that allow this formulation in a similar manner, e.g., gating \cite{ehler2011shrinkage}, or modulo sampling \cite{bhandari2020mod}. In Section \ref{sec:bounds}, we prove Theorem 1 on the optimal lower Lipschitz bounds for ReLU layers and saturation (Corollary \ref{cor:lip} and \ref{cor:lip2}). 
We show as well how our estimates correspond nicely to the known lower Lipschitz bounds for phase retrieval~\cite{bandeira2014savingphase}.  Notably, even though we obtain directly analogous lower Lipschitz bounds for all three activation functions, our proof for ReLU and $\lambda$-saturation is very different from the previously known proof for phase retrieval. 

\section{Injectivity of non-linear measurements}\label{sec:2}
For the derivation of the optimal lower Lipschitz bounds for ReLU layers and saturated measurements we make explicit use of the injectivity characterizations of the associated non-linear measurement operators given in~\cite{alharbi2024sat,haider2024relu,puthawala2022relu}. In a nutshell, we obtain injectivity if and only if for any input there are sufficiently many coordinates in the output that are not affected by the activation. In this section, we recall these results as a preparation and find that also the injectivity characterization for phase retrieval can be formulated in a similar manner. This emphasizes the proposed unified perspective as non-linear measurement operators.

\subsection{Frames and linear measurements}
We will use standard language and tools  from frame theory, and  we recommend ~\cite{casazza2012finiteframes,christensen2003frames} as references.
A collection of vectors $(\phi_i)_{i\in I}$ in $\RR^n$ with $|I|=m\geq n$ is called a \textit{frame} for $\RR^n$ if there exist constants $0<A\leq B<\infty$ with
\begin{equation}\label{eq:frame}
    A\|x\|^2\leq \sum_{i\in I}\vert \xphi \vert^2 \leq B\| x\|^2\hspace{.2cm}\textrm{ for all }x\in \RR^n.
\end{equation}
The vectors $\phi_i$ are called \textit{frame elements} and the inner products, or measurements $\xphi$ are called \textit{frame coefficients} for $x$. The bounds $A$ and $B$ are called lower and upper frame bounds, respectively. If an inequality is sharp then the corresponding bound is called an \textit{optimal frame bound}. An equivalent formulation of \eqref{eq:frame} is that the analysis operator associated with $(\phi_i)_{i\in I}$, given by
\begin{align}
\begin{split}
    \Theta:\RR^n&\rightarrow \mathbb{R}^m\\
    x&\mapsto\left(\langle x,\phi_i\rangle\right)_{i\in I},
\end{split}
\end{align}
is one-to-one and the optimal frame bounds coincide with the smallest and largest eigenvalues of the frame operator $\Theta^*\Theta$. Within the paradigm of this paper, the analysis operator is the \textit{linear measurement operator}, i.e., measurement without activation.

\subsection{Injectivity of ReLU layers and saturation recovery} For a frame $(\phi_i)_{i\in I}$, a bias vector $\alpha\in \RR^m$, and the ReLU function defined in \eqref{eq:reludef}, the non-linear measurement operator that corresponds to a ReLU layer is given by
\begin{align}
\begin{split}
    \Ca:\RR^n&\rightarrow \mathbb{R}^m\\
    x&\mapsto\left(\relu(\langle x,\phi_i\rangle-\alpha_i)\right)_{i\in I}.
\end{split}
\end{align}
To ensure injectivity of $\Ca$ we need to check if for any input $x\in \RR^n$, the set of frame elements that are not affected by ReLU -- i.e., \textit{activated} -- forms a frame. Hence, we are interested in those frame elements that satisfy $\xphi - \alpha_i\geq 0$.
%Similar to the sets in \eqref{eq:I+} and \eqref{eq:I-}, we
Denoting the set of coordinates that get activated for an $x\in \RR^n$ as
\begin{equation}
    \I(x) = \{i\in I: \xphi \geq \alpha_i\},
\end{equation}
we have the following characterization of the injectivity of $\Ca$~\cite{haider2023relu,haider2024relu,puthawala2022relu}.
\begin{theorem}[Puthawala et al. 2022, Haider et al., 2024]\label{thm:reluinj}
Let $(\phi_i)_{i\in I}$ be a frame for $\RR^n$ and $\alpha\in\RR^m$ a fixed bias. The following are equivalent.
\begin{enumerate}[(i)]
    \item The ReLU layer $\Ca$ is one-to-one.
    \item For all $x\in \RR^n$, the activated collection $(\phi_i)_{i\in \I(x)}$ is a frame for $\RR^n$.
\end{enumerate}
%A ReLU layer $\Ca$ is injective if and only if for any $x\in \RR^n$ the activated collection $(\phi_i)_{i\in \I(x)}$ is a frame for $\RR^n$.
\end{theorem}

For $\lambda$-saturation, we ask for the injectivity of the saturated measurement operator, which, for a saturation level $\lambda>0$ and the saturation function $\sigma_{\lambda}$ defined in \eqref{eq:satdef} is given by
\begin{align}
\begin{split}
    \Sl:\mathbb{B}_{\RR^n}&\rightarrow \mathbb{R}^m\\
    x&\mapsto\left(\sigma_{\lambda}(\xphi)\right)_{i\in I}.
\end{split}
\end{align}
%The activated coordinates are those
Analogous to the ReLU case, for any $x\in\mathbb{B}_{\RR^n}$ we have to check if the set of frame elements that are not affected by the saturation function -- i.e., \textit{activated} -- forms a frame. For this, the magnitudes of the frame coefficients has to lie below the saturation level $\lambda$. Denoting the set of all non-saturated coordinates for $x$ as
\begin{equation}
    \Il(x) = \{i\in I: \vert\xphi \vert \leq \lambda\},
\end{equation}
we have the following characterization~\cite{alharbi2024sat}.

\begin{theorem}[Alharbi et al., 2024]\label{thm:satinj}
Let $(\phi_i)_{i\in I}$ be a frame for $\RR^n$ and $\lambda>0$ be a fixed saturation level. The following are equivalent.
\begin{enumerate}[(i)]
    \item The saturated measurement operator $\Sl$ is one-to-one on $\mathbb{B}_{\RR^n}$.
    \item For all $x\in \mathbb{B}_{\RR^n}$, the unsaturated collection $(\phi_i)_{i\in \Il(x)}$ is a frame for $\RR^n$.
\end{enumerate}
\end{theorem}

\subsection{An excursion to other non-linear measurement}
In both of the above discussed cases, the characterization of the injectivity reads as that the activated (input-dependent) sub-collections of the measurement vectors $(\phi_i)_{i\in I}$ need to be frames. This argument can be naturally transferred to other suitable activations functions, too. We demonstrate this for the complementary setting to $\lambda$-saturation, i.e., where measurements are only recorded accurately if they are larger than a certain threshold $\mu>0$. The corresponding \textit{gating} activation function is given by $\gamma_{\mu}(t)=t$ if $|t|\geq \mu$ and $\gamma_{\mu}(t)=0$ if $|t|< \mu$.
% \begin{align}\label{eq:satdef}
%     \gamma_{\mu}(t) = \operatorname{sign}(t)\cdot\max\left(\vert t\vert ,\mu\right).
% \end{align}
Since very small vectors will never be gated -- i.e., \textit{activated} -- we consider $X=\RR^n\setminus r\mathbb{B}_{\RR^n}$ for $r>0$ as an input domain here. By a scaling argument, we can without loss of generality assume $r=1$. The gated measurement operator is given by
\begin{align}
\begin{split}
    G_\mu:\RR^n\setminus \mathbb{B}_{\RR^n}&\rightarrow \mathbb{R}^m\\
    x&\mapsto\left(\gamma_{\mu}(\xphi)\right)_{i\in I}.
\end{split}
\end{align}
Letting $I_\mu(x)=\{i\in I: \vert\xphi \vert \geq \mu\}$, we can use the same proof techniques as in \cite{alharbi2024sat} and \cite{haider2024relu} to show that for a fixed gating level $\mu>0$, the gated measurement operator $G_\mu$ is one-to-one on $\RR^n\setminus\mathbb{B}_{\RR^n}$ if and only if for any $x\in \RR^n\setminus \mathbb{B}_{\RR^n}$, the gated collection $(\phi_i)_{i\in I_\mu(x)}$ is a frame for $\RR^n$. The key here is that the domain $\RR^n\setminus \mathbb{B}_{\RR^n}$ is an open set in $\RR^n$.\\
% \begin{theorem}\label{thm:gatinj}
% Let $(\phi_i)_{i\in I}$ be a frame for $\RR^n$ and $\mu>0$ be a fixed gating level. The following are equivalent.
% \begin{enumerate}[(i)]
%     \item The gated measurement operator $G_\mu$ is one-to-one outside of $\mathbb{B}_{\RR^n}$.
%     \item For any $x\in \RR^n\setminus \mathbb{B}_{\RR^n}$, the gated collection $(\phi_i)_{i\in I_\mu(x)}$ is a frame for $\RR^n$.
% \end{enumerate}
% \end{theorem}

For (real) phase retrieval, the characterizing property of the underlying frame to ensure injectivity is the famous complement property \cite{balan2006phase}. It turns out that we can reformulate it so that the characterization reads analogously to Theorems \ref{thm:reluinj}, \ref{thm:satinj}. We elaborate on this in the following.

\subsection{Injectivity of phase retrieval - revisited}\label{sec:prinj} In phase retrieval, the non-linear measurement operator is the \textit{intensity measurement operator}, arising as
\begin{align}
\begin{split}
    \M:\RR^n/_\sim&\rightarrow \mathbb{R}^m\\
    [x]_\sim&\mapsto\left(|\langle x,\phi_i\rangle|\right)_{i\in I}.
\end{split}
\end{align}
It is well-known that $\M$ is one-to-one if and only if for any subset $J\subseteq I$, either $(\phi_i)_{i\in J}$ or $(\phi_i)_{i\in J^c}$ is a frame for $\RR^n$. This condition is called the \textit{complement property}.
%, and has been very useful when studying phase retrieval in frame theory.
While the complement property considers all partitions $\{J,J^c\}$ of $I$, in the proof of Theorem 2.8 in \cite{balan2006phase} where the equivalence was first proven, only partitions of a specific form are considered.
%These sets are determined by the nature of the metric $d(x,y)=\min\{\Vert x-y \Vert, \Vert x+y \Vert\}$ that is used in $\RR/_\sim$ and consider for any pairs of vectors $x,y\in \RR^n$
In particular, one has to ensure that for any pair $x,y\in \RR^n$ either the collection of frame elements where the frame coefficients for $x$ and $y$ have the same signs is a frame, or the collection where they have opposite signs is a frame.
We shall denote the corresponding sets of coordinates as
\begin{align}
    \label{eq:I+}
    I^+(x,y) &= \{i\in I: \xphi \yphi \geq 0\},\\
    \label{eq:I-}
    I^-(x,y) &= \{i\in I: \xphi \yphi \leq 0\}.
\end{align}
%Although the complement property considers all partitions $\{J,J^c\}$ of $I$, it follows from the proof of Theorem 2.8 in \cite{balan2006phase} that one actually only needs to consider partitions of the form $\{I^+(x,y),I^-(x,y)\}$ for $x,y\in\RR^n$.
Note that if $\langle x,\phi_i\rangle\langle y,\phi_i\rangle=0$ for some $i\in I$ then
%$I^+(x,y)$ and $I^-(x,y)$ are not disjoint and
$\{I^+(x,y),I^-(x,y)\}$ is not a partition of $I$. This technicality does not interfere with the result though as it occurs on a nowhere dense subset.
Using this perspective, we may reformulate the injectivity characterization of $\M$ stated in~\cite{balan2006phase} in the following sense.
\begin{theorem}[Reformulation of Balan et al., 2006]\label{thm:phaseinj} Let $(\phi_i)_{i\in I}$ be a frame for $\RR^n$.
The following are equivalent.
\begin{enumerate}
    \item[(i)] The intensity measurement operator $\M$ is one-to-one.
    \item[(ii)] For any $J\subseteq I$, either $(\phi_i)_{i\in J}$ or $(\phi_i)_{i\in J^c}$ is a frame for $\RR^n$.
    \item[(ii*)] For any $x,y\in \RR^n$, either $(\phi_i)_{i\in I^+(x,y)}$ or $(\phi_i)_{i\in I^-(x,y)}$ is a frame for $\RR^n$.
\end{enumerate}
\end{theorem}

Point $(ii^*)$ now says that injectivity of $\M$ can be determined by checking the frame condition for specific sub-collections of $(\phi_i)_{i\in I}$ that depend on $x,y$ - just as for the injectivity of ReLU layers and saturation recovery. While this demonstrates that the conditions for the injectivity of the non-linear measurement operators are of similar nature for all three problems, the one for phase retrieval stands out in a particular way as it depends on one of two sets of vectors being a frame.
The reason for this is that the intensity measurement operator $\M$ is defined on the quotient space $\RR^n/_\sim$ and we need to compensate for the problem of not a priori knowing whether we wish to compare $x$ to $y$ or $x$ to $-y$. Hence, to prove that $\M$ is injective on $\RR^n/_\sim$, we need to ensure that if $\M [x]_\sim=\M [y]_\sim$ for some $x,y\in \RR^n$ then $x=y$ or $x=-y$. If  $\M [x]_\sim=\M [y]_\sim$ and $(\phi_i)_{i\in I^+(x,y)}$ is a frame then it follows that $x=y$. Otherwise, if $\M [x]_\sim=\M [y]_\sim$ and $(\phi_i)_{i\in I^-(x,y)}$ is a frame then $x=-y$. This also contributes to the difficultly in building a reconstruction map for phase retrieval ~\cite{prusa2017phasereconstruction, candes2013phaselift, nenov2024faster}.  
On the other hand, in  the case of ReLU layers and saturation recovery, the activated frame elements can not only be used to determine injectivity but also to build a locally linear reconstruction map~\cite{alharbi2024sat, haider2024relu}.

\section{Optimal lower Lipschitz bounds}\label{sec:bounds}
This section represents the main part of the paper. We will derive optimal lower Lipschitz bounds for ReLU layers and saturated measurements, and also present a slightly improved version of the known stability result for phase retrieval.

\subsection{Stability of ReLU layers}\label{sec:relu}
From the characterization in Theorem \ref{thm:reluinj} we can immediately deduce that if a ReLU layer $\Ca$ is one-to-one on $\RR^n$ then every activated collection $(\phi_i)_{i\in \I(x)}$ has a positive lower frame bound. We shall refer to the optimal lower frame bound of $(\phi_i)_{i\in \I(x)}$ as $A_\alpha(x)$. Since there are only finitely many different combinations of activated weight vectors, there is a vector $x^*\in\RR^n$ such that $A_\alpha(x^*)\leq A_\alpha(x)$ for all $x\in \RR^n$. We denote this value as
\begin{equation}
    A_\alpha=\min_{x\in \RR^n} A_\alpha(x).
\end{equation}
Note that for any injective ReLU layer $\Ca$, we can find an open set $U\subseteq\RR^n$ so that $\I(x)=\I(y)$ for all $x,y\in U$ and that $A_\alpha=A_{\alpha(x)}$ for all $x\in U$.  We have that $\Ca$ is affine linear on $U$, and as $U$ is open,
$A_\alpha$ is the largest number so that $A_\alpha\|x-y\|^2\leq \Vert \Ca(x)-\Ca(y) \Vert^2$ for all $x,y\in U$.
Therefore, no lower Lipschitz bound can exceed $\sqrt{A_\alpha}$.

The following theorem provides a lower Lipschitz bound that is specific to the input vectors $x,y\in \RR^n$.
In \cite{puthawala2022relu}, the authors consider the line segment from $x$ to $y$ and use that it crosses through at most $m$ polytopes where $C_\alpha$ is linear.  Because of this, their lower Lipschitz bound is dependent on $m$.  We use a different argument which only considers the midpoint $\frac{x+y}{2}$.   This allows for a lower Lipschitz bound for $\Ca$ which is a constant multiple of $\sqrt{A_\alpha}$. 

\begin{theorem}\label{thm:relu}
Let $(\phi_i)_{i\in I}$ be a frame for $\RR^n$ and $\alpha\in\RR^m$ a bias vector.  For all $x,y\in\RR^n$, we have that
    % \begin{equation}
    %     \tfrac{1}{2}A_\alpha^{\nicefrac{1}{2}}\; \|x-y\| \leq \tfrac{1}{2}A_\alpha({\tfrac{x+y}{2}})^{\nicefrac{1}{2}}\; \|x-y\| \leq \|\Ca(x)-\Ca(y)\|.
    % \end{equation}
    \begin{equation}
        %\tfrac{1}{4}A_\alpha\; \|x-y\|^2 \leq
        \tfrac{1}{4}A_\alpha({\tfrac{x+y}{2}})\; \|x-y\|^2 \leq \|\Ca(x)-\Ca(y)\|^2.
    \end{equation}
    %where $A({\frac{x+y}{2})}$ is the lower frame bound of $(\phi_i)_{i\in \I(\frac{x+y}{2})}$, i.e., the frame activated by $\tfrac{x+y}{2}$.\\
    % \begin{equation}
    %     \tfrac{A_\alpha^{\nicefrac{1}{2}}}{2}\; \|x-y\| \leq \|\Ca(x)-\Ca(y)\|,
    % \end{equation}    
    % $\left(\tfrac{A_0}{2}\right)^{\frac{1}{2}}$ with $A_0$ defined as above is a universal lower Lipschitz bound for $\Ca$.
\end{theorem} 

\begin{proof}
We first claim that
\begin{equation}\label{eq:relu1}
    (\I(x)\cap \I(y))\subseteq\I\left(\tfrac{x+y}{2}\right).
\end{equation}
%$(\I(x)\cap \I(y))\subseteq\I\left(\tfrac{x+y}{2}\right)$.  
Indeed, if  $i\in (\I(x)\cap \I(y))$ then $\langle x,\phi_i\rangle\geq \alpha_i$ and $\langle y,\phi_i\rangle\geq \alpha_i$.  Thus, $\langle x+ y,\phi_i\rangle\geq 2\alpha_i$, showing that $(\I(x)\cap \I(y))\subseteq\I\left(\tfrac{x+y}{2}\right)$.

We now prove that
\begin{equation}\label{eq:relu2}
    \I\left(\tfrac{x+y}{2}\right) \subseteq (\I(x)\cup \I(y)).
\end{equation}
%$\I\left(\tfrac{x+y}{2}\right) \subseteq (\I(x)\cup \I(y))$. 
Let  $i\in \I(\tfrac{x+y}{2})$ then $\langle x+ y,\phi_i\rangle\geq 2\alpha_i$.  Either $\langle x,\phi_i\rangle\geq \alpha_i$ or  $\langle y,\phi_i\rangle\geq \alpha_i$.  This gives that $i\in (\I(x)\cup\I(y))$.

Furthermore, note that if $i\in \I(\tfrac{x+y}{2})$ then
   $$
    2(\langle x, \phi_i\rangle -\alpha_i)\geq   2\langle x, \phi_i\rangle-\langle x+ y,\phi_i\rangle
    =\langle x-y, \phi_i\rangle.
    $$
If we have as well that $i\in \big(\I(x)\setminus \I(y)\big)\cap\I(\tfrac{x+y}{2})$ then both sides of the above inequality are positive.  Hence,  we have that
\begin{equation}\label{E:half}
    2|\langle x,\phi_i\rangle-\alpha_i|\geq |\langle x-y,\phi_i\rangle|\hspace{.5cm}\textrm{ for all }i\in \big(\I(x)\setminus \I(y)\big)\cap\I(\tfrac{x+y}{2}).
\end{equation}
We now obtain the following estimate for $x,y\in \RR^n$.

\begin{align*}
    \|\Ca(x)-\Ca(y)\|^2 &= \sum_{i\in \I(x)\setminus \I(y)}|\langle x,\phi_i\rangle -\alpha_i|^2+\sum_{i\in \I(y)\setminus \I(x)}|\langle y,\phi_i\rangle -\alpha_i|^2\\ &\qquad+\sum_{i\in \I(x)\cap \I(y)}|\langle x-y,\phi_i\rangle|^2\\
    &\geq\sum_{\substack{i\in (\I(x)\setminus \I(y))\\\cap\; \I(\frac{x+y}{2})}}\tfrac{1}{4}\;|\langle x-y,\phi_i\rangle|^2+\sum_{\substack{i\in (\I(y)\setminus \I(x))\\\cap\; \I(\frac{x+y}{2})}}\tfrac{1}{4}\;|\langle x-y,\phi_i\rangle|^2\\ &\qquad+\sum_{i\in \I(x)\cap \I(y)}\tfrac{1}{4}\;|\langle x-y,\phi_i\rangle|^2\hspace{.75cm}\textrm{ by }\eqref{E:half},
    \\
    & =\sum_{\substack{i \in (\I(x) \cup \I(y)) \\ \cap\; \I\left(\frac{x+y}{2}\right)}}\tfrac{1}{4}\;|\langle x-y,\phi_i\rangle|^2\hspace{.5cm}\textrm{ by }\eqref{eq:relu1}
    %(\I(x)\cap\I(y))\subseteq \I(\tfrac{x+y}{2})
    ,\\
    & =\sum_{i\in  \I(\frac{x+y}{2})}\tfrac{1}{4}\;|\langle x-y,\phi_i\rangle|^2\hspace{.75cm}\textrm{ by }\eqref{eq:relu2}
    %\I(\tfrac{x+y}{2})\subseteq (\I(x)\cup\I(y))
    ,\\
    & \geq\tfrac{1}{4}\; A_\alpha({\tfrac{x+y}{2}})\|x-y\|^2.
\end{align*}

\end{proof}
% \begin{ex}\label{ex:0}
%     Let $(\phi_i)_{i\in I}$ be any frame such that the associated $\relu$ layer with bias $\alpha=0$ is injective. Consider any $x\neq 0$ and $y=-x$, then the activated weight vectors for $x$ and $y$ are disjoint, i.e., $\I(x)\cap \I(y)=\emptyset$.
% \end{ex}
%We arrive at our main result for ReLU layers.

As $\kappa_L\leq \sqrt{A_\alpha}$, we have the following corollary which proves Theorem 1 for the case of ReLU layers.

\begin{corollary}\label{cor:lip}
Let $(\phi_i)_{i\in I}$ be a frame for $\RR^n$ and $\alpha\in\RR^m$ a bias vector such that the associated $\relu$ layer $\Ca$ is one-to-one.
    If $\kappa_L$ is the optimal lower Lipschitz bound of $\Ca$ then
    \begin{equation}
        \tfrac{1}{2}\sqrt{A_\alpha} \leq \kappa_L \leq \sqrt{A_\alpha}.
    \end{equation}
    % \begin{equation}
    %     \tfrac{1}{2}A_\alpha^{\nicefrac{1}{2}} \leq d \leq A_\alpha^{\nicefrac{1}{2}}.
    % \end{equation}
\end{corollary}

When Lipschitz bounds for ReLU layers were first being studied, it was conjectured that it was always the case that $\kappa_L=\sqrt{A_\alpha}$.  However,
an example is given in \cite{puthawala2022relu} where $\kappa_L<\sqrt{A_\alpha}$.  We extend this example to give for every $n\in\NN$ a simple construction of a frame for $\RR^n$ where $\kappa_L=\tfrac{1}{\sqrt{2}}\sqrt{A_\alpha}$.

\begin{proposition}\label{prop:ex}
    Let $(\phi_i)_{i\in I}$ be a frame for $\RR^n$ with optimal lower frame bound $A$. The following holds.
    \begin{enumerate}[(i)]
        \item The ReLU layer $\Ca$ associated with $(\phi_i)_{i\in I}\cup (-\phi_i)_{i\in I}$ and $\alpha = \boldsymbol{0}$ is one-to-one.
        \item $A_\alpha = A$.
        \item The optimal lower Lipschitz bound for $\Ca$ equals $\tfrac{1}{\sqrt{2}}\sqrt{A}$.
    \end{enumerate}
\end{proposition}    

\begin{proof}%[Proof of Prop. \ref{prop:ex}]
    Point $(i)$ is true since for any $x\in \RR^n$ it holds that either $\xphi \geq 0$ or $-\xphi\geq 0$. Hence, for all $i\in I$ either $\phi_i$ or $-\phi_i$ is activated. Since $(\pm\phi_i)_{i\in I}$ is still a frame for any combination of signs, the associated ReLU layer is injective by Theorem \ref{thm:reluinj}. Moreover, since changing signs of the frame elements does not change the lower frame bound, $(ii)$ immediately follows.
    
    %Hence, for any $x\in \RR^n$ the lower frame bound of $(\phi_i)_{i\in \I(x)}$ equals the lower frame bound of $(\phi_i)_{i\in I}$.
    
    We now prove $(iii)$. Let us denote $I^+(x,y) = \{i\in I: \xphi \yphi > 0\}$ and $I^-(x,y) = \{i\in I: \xphi \yphi \leq 0\}$. Furthermore, recall the fact that for any $a,b>0$ it holds that
    \begin{equation}\label{lem:ab}
        a^2+b^2\geq \tfrac{1}{2}(a+b)^2
    \end{equation}
    with equality if $a=b$.
    Using this, for any $x,y\in \RR^n$ we deduce
\begin{align*}
    \|\Ca(x)-\Ca(y)\|^2 
    &= \sum_{i\in I^+(x,y)}  | \langle x, \phi_i\rangle -\langle y, \phi_i\rangle|^2 + \sum_{i\in I^-(x,y)} |\langle x, \phi_i\rangle|^2 + |\langle y, \phi_i\rangle|^2\\
    &\geq \sum_{i\in I^+(x,y)}  | \langle x, \phi_i\rangle -\langle y, \phi_i\rangle|^2 + \tfrac{1}{2}\sum_{i\in I^-(x,y)}  | \langle x, \phi_i\rangle -\langle y, \phi_i\rangle|^2
    %\hspace{.5cm}\textrm{ by \eqref{lem:ab}},
    \\
    &\geq \tfrac{1}{2} \sum_{i\in I}  | \langle x, \phi_i\rangle -\langle y, \phi_i\rangle|^2\\
    &\geq \tfrac{1}{2} A \Vert x-y \Vert^2.
\end{align*}
Thus, we have that $\kappa_L\geq \tfrac{1}{\sqrt{2}}\sqrt{A}$.
Now let $u\in\RR^n$ be a unit vector that satisfies $A \Vert u \Vert^2 = \sum_{i\in I} \vert \langle u, \phi_i\rangle \vert^2$. Setting $x=u$ and $y=-u$, we have that $I^+(x,y)=\emptyset$ and $| \langle x, \phi_i\rangle|^2 =|\langle y, \phi_i\rangle|^2$. With this we obtain
\begin{align*}
    \|\Ca(x)-\Ca(y)\|^2
    &= \sum_{i\in I^-(x,y)} |\langle x, \phi_i\rangle|^2 + |\langle y, \phi_i\rangle|^2\\
    &= \sum_{i\in I} |\langle x, \phi_i\rangle|^2 + |\langle y, \phi_i\rangle|^2\hspace{.5cm}\textrm{ by invariance of the signs},\\
    &= \tfrac{1}{2}\sum_{i\in I}  | \langle x, \phi_i\rangle -\langle y, \phi_i\rangle|^2 \hspace{.5cm}\textrm{ by equality in \eqref{lem:ab}},
    \\
    &= \tfrac{1}{2}\sum_{i\in I}  | \langle 2u, \phi_i\rangle|^2 \hspace{.5cm}\textrm{ by inserting }x=u, y=-u,
    \\
    &= \tfrac{1}{2}A \Vert 2u \Vert^2 \hspace{.5cm}\textrm{ by design of }u,
    \\
    &= \tfrac{1}{2} A \Vert x-y \Vert^2 \hspace{.5cm}\textrm{ as }2u = x-y.
\end{align*}
This gives that $\kappa_L\leq \tfrac{1}{\sqrt{2}}\sqrt{A}$ and completes the proof.
\end{proof}
% Let $\varepsilon>0$ and choose $\gamma >0$ such that $ \gamma>\max_{\substack{i\in I}} |\frac{\alpha_i}{\varepsilon\langle x^*,\phi_i \rangle}|$.
% Now set $x=\gamma x^*$ and $y=-\gamma x^*$. It follows that $\I(x)=\I(x^*)$ and $\I(y)=\I(-x^*)=I\setminus\I(x^*)$. In particular, $\I(x)\cap \I(y)=\emptyset$ holds. With that, we deduce
% \begin{align*}
%     \|&\Ca(x)-\Ca(y)\|^2 = \sum_{i\in \I(x^*)} \vert \langle \gamma x^*, \phi_i\rangle -\alpha_i\vert^2 + \sum_{i\in \I(-x^*)} \vert \langle -\gamma x^*, \phi_i\rangle-\alpha_i \vert^2 \\
%     &= \sum_{i\in \I(x^*)} \vert \langle \gamma x^*, \phi_i\rangle -\tfrac{\alpha_i}{\gamma \langle  x^*, \phi_i\rangle}\langle \gamma x^*, \phi_i\rangle \vert^2 + \sum_{i\in \I(-x^*)} \vert \langle -\gamma x^*, \phi_i\rangle-\tfrac{\alpha_i}{\gamma \langle  -x^*, \phi_i\rangle}\langle -\gamma x^*,\phi_i\rangle\vert^2 \\
%             &> (1-\varepsilon)^2 \sum_{i\in \I(x^*)} \vert \langle \gamma x^*, \phi_i\rangle \vert^2 + (1-\varepsilon)^2\sum_{i\in \I(-x^*)} \vert \langle -\gamma x^*, \phi_i\rangle\vert^2 \\
%     &\geq (1-\varepsilon)^2 \left(\sum_{i\in \I(x^*)}|\langle \gamma x^*,\phi_i\rangle|^2+\sum_{i\in \I(-x^*)}|\langle \gamma x^*,\phi_i\rangle|^2\right)\\
%     &\geq 2 (1-\varepsilon)^2 A_\alpha \Vert \gamma x^* \Vert^2\\
%     &= \tfrac{1}{2} (1-\varepsilon)^2 A_\alpha \Vert x-y \Vert^2\hspace{.25cm}\textrm{ as }x-y=2\gamma x^*.
% \end{align*}
% We obtain equality if $\alpha = 0$, and we can choose $u=x^*$ so that $A_\alpha(x^*) = A_\alpha(-x^*)$. This is\\

For the special construction of an injective ReLU layer as in the above proposition, it is easy to explicitly compute the optimal lower Lipschitz bound for $\Ca$ using a vector where the activated coordinates are disjoint. In a more general case, computing an optimal lower bound is significantly more difficult. Moreover, the bound that we obtain in Proposition \ref{prop:ex} is still off by a factor of $\tfrac{1}{\sqrt{2}}$ to the bound that we derived in Corollary \ref{cor:lip}.
In the context of neural networks, knowing the best possible estimate for the optimal lower Lipschitz bound for the individual layers is crucial since any errors propagate under the composition of multiple layers. Therefore, our improved lower Lipschitz bound for single ReLU layers becomes even more important when one considers deeper neural networks. Although the numerical experiments in \cite{RauschMasters} give strong evidence that our factor of $\tfrac{1}{2}$ is indeed optimal, we are left with the following open problem.

\begin{problem}
    What is the largest $ \tfrac{1}{2}\leq K\leq\tfrac{1}{\sqrt{2}}$ so that for every frame $(\phi_i)_{i\in I}$ and a bias vector $\alpha\in\RR^m$ such that the associated ReLU layer $\Ca$ is one-to-one,  its optimal lower Lipschitz bound satisfies $\kappa_L\geq K\sqrt{A_\alpha}$?
\end{problem}

\subsection{Stability of saturated measurements}\label{sec:sat}
Since the condition for recovering a vector from saturated measurements is very closely related to the injectivity of a ReLU layer, it is not surprising that the nature of their lower Lipschitz stability is closely related. As a result,  the same technique as in the proof of Theorem \ref{thm:relu} allows us to to derive an explicit lower Lipschitz bound for the saturation recovery problem.
%As an immediate consequence, the open problem of whether saturation recovery is stable at the critical saturation level is solved.\\

Note that if the saturated measurement operator $\Sl$ is one-to-one on $\mathbb{B}_{\RR^n}$ then for every $x\in \mathbb{B}_{\RR^n}$ we have that  $(\phi_i)_{i\in \Il(x)}$ is a frame of $\RR^n$.  We denote $A_\lambda(x)$ to be the optimal lower frame bound for $(\phi_i)_{i\in \Il(x)}$ and   denote
\begin{equation}
    A_\lambda=\min_{x\in \mathbb{B}_{\RR^n}} A_\lambda(x).
\end{equation}
 In \cite{alharbi2024sat}, it is proven that if $\lambda$ is greater than some critical threshold then $\Sl$ is injective on $\mathbb{B}_{\RR^n}$ if and only if it is bi-Lipschitz.  Furthermore, as is the case for ReLU layers, the lower Lipschitz bound satisfies $\kappa_L\leq A_\lambda$.  However, the proof does not provide a  lower bound on $\kappa_L$. 

For the presentation of our estimation it will be convenient to
%for individual $x,y\in \RR^n$ in Theorem \eqref{thm:sat}, and the global statement in the subsequent corollary,
introduce some more notation. First, when it comes to saturated coordinates, it is crucial to distinguish between positively and negatively saturated coordinates. For $x\in \mathbb{B}_{\RR^n}$, we set
\begin{align}
    \Il^+(x) &= \{i\in I: \xphi > \lambda\},\\
    \Il^-(x) &= \{i\in I: \xphi < \lambda\}.
\end{align}
When considering the situation where for two vectors $x,y\in \mathbb{B}_{\RR^n}$ the associated coordinates are both saturated, i.e., $|\langle x,\phi_i \rangle|> \lambda$ and $|\langle y,\phi_i \rangle|> \lambda$, there are two possibilities that can occur to the corresponding frame coefficients:
\begin{itemize}
    \item If they have the same sign, $i\in (\Il^+(x)\cap \Il^+(y))\cup (\Il^-(x)\cap \Il^-(y))$, then %$\operatorname{sign}\langle x,\phi_i \rangle=\operatorname{sign}\langle y,\phi_i \rangle$ then
    $ |\sigma_\lambda(\langle x,\phi_i \rangle)-\sigma_\lambda(\langle y,\phi_i \rangle)|=0.$
    \item If they have opposite signs, $i\in (\Il^+(x)\cap \Il^-(y))\cup (\Il^-(x)\cap \Il^+(y))$, then
    %$\operatorname{sign}\langle x,\phi_i \rangle=-\operatorname{sign}\langle y,\phi_i \rangle$ then 
    $|\sigma_\lambda(\langle x,\phi_i \rangle)-\sigma_\lambda(\langle y,\phi_i \rangle)|=2\lambda.$
\end{itemize}
% \begin{equation}
%     \textrm{ if }\quad\operatorname{sign}\langle x,\phi_i \rangle=\operatorname{sign}\langle y,\phi_i \rangle\quad\textrm{ then }\quad |\sigma_\lambda(\langle x,\phi_i \rangle)-\sigma_\lambda(\langle y,\phi_i \rangle)|=0.
% \end{equation}
% \begin{equation}\label{eq:oppositesat}
%     \textrm{ if }\quad\operatorname{sign}\langle x,\phi_i \rangle=-\operatorname{sign}\langle y,\phi_i \rangle\quad\textrm{ then }\quad|\sigma_\lambda(\langle x,\phi_i \rangle)-\sigma_\lambda(\langle y,\phi_i \rangle)|=2\lambda.
% \end{equation}
For this reason, the latter case needs to be treated differently as it plays a special role in the estimate. We shall define the corresponding set of coordinates
%where the measurements for $x,y$ are saturated with opposite sign 
as
\begin{equation}
    \Il^{\Delta}(x,y) = \left( \Il^+(x) \cap \Il^-(y) \right) \cup \left( \Il^-(x) \cap \Il^+(y) \right).
\end{equation}
Now we can formulate our result on the lower Lipschitz bound of $\Sl$  that depends on $x,y\in \RR^n$, analogous to Theorem \ref{thm:relu}. The global statement follows in Corollary \ref{cor:lip2}.

\begin{theorem}\label{thm:sat}
Let $(\phi_i)_{i\in I}$ be a frame for $\RR^n$ and $\lambda>0$. 
    %Assume $(\phi_i)_{i\in I}$ does $\lambda$-saturation recovery on $\mathbb{B}_{\RR^n}$.
    For all $x,y\in\mathbb{B}_{\RR^n}$, we have that
    % \begin{equation}
    %     \left(\tfrac{1}{4}A_\lambda^\Delta(x,y) + \lambda^2\cdot \vert \Il^{\Delta}(x,y) \vert\right)^{\nicefrac{1}{2}}\; \|x-y\| \leq \|\Sl(x)-\Sl(y)\|,
    % \end{equation}
    \begin{equation}
        \left(\tfrac{1}{4}A_\lambda^\Delta(x,y) + \vert \Il^{\Delta}(x,y) \vert \cdot \lambda^2\right)\; \|x-y\|^2 \leq \|\Sl(x)-\Sl(y)\|^2,
    \end{equation}
    where $A_\lambda^\Delta(x,y)$ is the optimal lower frame bound of $(\phi_i)_{i\in \Il(\frac{x+y}{2}) \setminus \Il^{\Delta}(x,y)}$.% it is forms a frame, and $0$ if it does not.
    % \begin{equation}
    %     \min\left\{
    %     \tfrac{A_0^{\nicefrac{1}{2}}}{2}, 2\lambda\right\}\; \|x-y\| \leq \|\Sl(x)-\Sl(y)\|.
    % \end{equation}
\end{theorem}
It is important to note that if there are no saturated coordinates with opposite signs, i.e., $\Il^{\Delta}(x,y)=\emptyset$ then $A_\lambda^\Delta(x,y) = A_\lambda(\tfrac{x+y}{2})>0$ is the lower bound. On the other hand, if $\Il^{\Delta}(x,y)\neq\emptyset$ then $(\phi_i)_{i\in \Il(\frac{x+y}{2}) \setminus \Il^{\Delta}(x,y)}$ might not be a frame as more frame elements are removed. In this case, it might happen that $A_\lambda^\Delta(x,y)=0$, and then the lower bound depends on the number of saturated coordinates with opposite signs, given by $\vert \Il^{\Delta}(x,y) \vert$.

\begin{proof}
    By the same argument as for the ReLU case, we have that 
    \begin{equation}\label{eq:sat1}
        (\Il(x)\cap\Il(y))\subseteq \Il(\tfrac{x+y}{2}).
    \end{equation}
    Similarly, it holds that
    \begin{equation}\label{eq:sat2}
        \left(\Il\left(\tfrac{x+y}{2}\right)\setminus \Il^\Delta(x,y) \right)\subseteq \left(\Il(x)\cup \Il(y)\right).
    \end{equation}
    %Indeed, if $i\in (\Il\left(\tfrac{x+y}{2}\right)\setminus \Il^\Delta(x,y))$ then $|\langle x,\phi_i \rangle| + |\langle y,\phi_i \rangle|\leq 2\lambda$. Hence, either $i\in \Il(x)$ or $i\in \Il(y)$.
    
    Now assume that $i\in \big(\Il(x)\setminus \Il(y)\big)\cap\Il(\tfrac{x+y}{2})$ and observe the following for the two possible cases of saturation for $\langle y,\phi_i \rangle$.
    If $i\in \Il^+(y)$, then
    $
        2\vert \langle x, \phi_i\rangle - \lambda\vert\geq  
        \vert \langle x-y, \phi_i\rangle\vert,
    $
    and if $i\in \Il^-(y)$, then
    $
        2\vert \langle x, \phi_i\rangle +\lambda\vert\geq  
        \vert \langle x-y, \phi_i\rangle\vert.
    $
    So in general, we have for all $i\in \big(\Il(x)\setminus \Il(y)\big)\cap\Il(\tfrac{x+y}{2})$ that 
    \begin{equation}\label{E:2sat}
    2\vert \langle x, \phi_i\rangle - \operatorname{sign}( \langle y, \phi_i\rangle )\lambda\vert\geq \vert \langle x-y, \phi_i\rangle\vert .    
    \end{equation}
    
    %To make the notation more efficient we will write $\vert \langle x, \phi_i\rangle \pm \lambda\vert$, implicitly referring to the respective cases.
    Using this, we now can deduce the following estimate for $x,y\in \mathbb{B}_{\RR^n}$.
    \begin{align*}
    &\|\Sl(x)-\Sl(y)\|^2\\
    &=\sum_{i\in \Il(x)\setminus \Il(y)}|\langle x,\phi_i\rangle -\operatorname{sign}( \langle y, \phi_i\rangle )\lambda|^2 +\sum_{i\in \Il(y)\setminus \Il(x)}|\langle y,\phi_i\rangle -\operatorname{sign}( \langle x, \phi_i\rangle )\lambda|^2\\
    &\qquad +\sum_{i\in \Il(x)\cap \Il(y)}|\langle x-y,\phi_i\rangle|^2 +\sum_{i\in \Il^\Delta(x,y)}4\lambda^2\\
    &\geq\sum_{\substack{i\in (\Il(x)\setminus \Il(y))\\\cap\; \Il(\frac{x+y}{2})}}\tfrac{1}{4}\; |\langle x-y,\phi_i\rangle|^2+\sum_{\substack{i\in (\Il(y)\setminus \Il(x))\\\cap\; \Il(\frac{x+y}{2})}}\tfrac{1}{4}\; |\langle x-y,\phi_i\rangle|^2\\
    & \qquad + \sum_{i\in \Il(x)\cap \Il(y)}\tfrac{1}{4}\;|\langle x-y,\phi_i\rangle|^2 + \vert \Il^{\Delta}(x,y) \vert\cdot 4\lambda^2\hspace{.5cm}\textrm{ by }\eqref{E:2sat},\\
    & =
    \sum_{\substack{i \in (\Il(x) \cup \Il(y)) \\ \cap\; \Il\left(\frac{x+y}{2}\right)}}\tfrac{1}{4}\; |\langle x-y,\phi_i\rangle|^2 + \vert \Il^{\Delta}(x,y) \vert\cdot 4\lambda^2 \hspace{.5cm}\textrm{ by }\eqref{eq:sat1}%(\Il(x)\cap\Il(y))\subseteq \Il(\tfrac{x+y}{2})
    ,\\
    & =\sum_{i\in  \Il(\frac{x+y}{2})\setminus \Il^{\Delta}(x,y) }\tfrac{1}{4}\; |\langle x-y,\phi_i\rangle|^2 + \vert \Il^{\Delta}(x,y) \vert\cdot 4\lambda^2 \hspace{.5cm}\textrm{ by }\eqref{eq:sat2}
    %\left(\Il\left(\tfrac{x+y}{2}\right)\setminus \Il^\Delta(x,y) \right)\subseteq \left(\Il(x)\cup \Il(y)\right)
    ,\\
    & \geq\tfrac{1}{4}A_\lambda^\Delta(x,y)\|x-y\|^2 + \vert \Il^{\Delta}(x,y) \vert\cdot 4\lambda^2.
    %\text{if } (\phi_i)_{\Il(\frac{x+y}{2}) \setminus \Il^{\Delta}(x,y)} \text{ forms a frame,}
    \\   
    & \geq\left(\tfrac{1}{4}A_\lambda^\Delta(x,y) + \vert \Il^{\Delta}(x,y) \vert\cdot \lambda^2\right) \|x-y\|^2\hspace{.5cm}\textrm{ as }2\geq\|x-y\|.
\end{align*}
\end{proof}
The following corollary gives Theorem 1 for the case of $\lambda$-saturation.
\begin{corollary}\label{cor:lip2}
Let $(\phi_i)_{i\in I}$ be a frame for $\RR^n$ and $\lambda>0$ such that  $\Sl$ is one-to-one on $\mathbb{B}_{\RR^n}$.    If $\kappa_L$ is the optimal lower Lipschitz bound of $\Sl$ then
    \begin{equation}
        \min\left\{\tfrac{1}{2}\sqrt{A_\lambda}, \lambda\right\} \leq \kappa_L \leq \sqrt{A_\lambda}.
    \end{equation}
    % \begin{equation}
    %     \min\left\{\tfrac{1}{2}A_\lambda^{\nicefrac{1}{2}}, 2\lambda\right\} \leq d \leq A_\lambda^{\nicefrac{1}{2}}.
    % \end{equation}
\end{corollary}

\begin{proof}
    %Assume that $(\phi_i)_{i\in I}$ does $\lambda$-saturation recovery on $\mathbb{B}_{\RR^n}$ for some $\lambda\in \RR$.
    %Assume $\Sl$ is one-to-one on $\mathbb{B}_{\RR^n}$ for $\lambda\in \RR$.
    Recall that the upper bound $\kappa_L\leq \sqrt{A_\lambda}$ was proven in \cite{alharbi2024sat}.
    The lower bound follows from Theorem \ref{thm:sat} by the following argument. Let $x,y\in\mathbb{B}_{\RR^n}$.  We first assume that 
    $|\Il^{\Delta}(x,y)|=0$ and hence $A_\lambda^\Delta(x,y) = A_\lambda(\tfrac{x+y}{2})\geq A_\lambda$. Theorem \ref{thm:sat} gives that
    $$\tfrac{1}{2}\sqrt{A_\lambda}\|x-y\|\leq \|\Sl(x)-\Sl(y)\|.$$

    We now assume that $|\Il^{\Delta}(x,y)|\geq 1$.  In this case $(\phi_i)_{i\in \Il(\frac{x+y}{2}) \setminus \Il^{\Delta}(x,y)}$ might not form a frame.  The worst case is that $A_\lambda^\Delta(x,y)=0$ and $|\Il^{\Delta}(x,y)|= 1$. Theorem \ref{thm:sat} then gives that
    $$\lambda\|x-y\|\leq \|\Sl(x)-\Sl(y)\|.$$

    Thus, we have that $\min\left\{\tfrac{1}{2}\sqrt{A_\lambda}, \lambda\right\} \leq \kappa_L$.
\end{proof}

% In the case of unit norm equiangular tight frames in $\RR^n$ with $n+1$ elements this value is equal to
% \begin{equation}
%     \lambda_c=\sqrt{\tfrac{1+n}{2}}.
% \end{equation}

In the case of using the ReLU activation function, we obtained an optimal lower bound (up to a factor of at most $\tfrac{1}{\sqrt{2}}$) for the lower Lipschitz bound of ReLU layer.  For the case of $\lambda$-saturation, we have a lower bound for the lower Lipschitz bound of the $\lambda$-saturated measurement operator of $\min\{\tfrac{1}{2}\sqrt{A_\lambda},\lambda\}$.  The lower bound necessarily depends on both $A_\lambda$ and $\lambda$.
Note that $\tfrac{1}{2}\sqrt{A_\lambda}$ is optimal up to a constant factor, but it is not clear what the dependence on $\lambda$ should be. We are left with the following open problem.

% \begin{problem}
%     What is the largest $K \geq \lambda$ so that for every frame $(\phi_i)_{i\in I}$ and a saturation level $\lambda$ such that the associated  $\Ca$ is one-to-one,  its optimal lower Lipschitz bound satisfies $\kappa_U\geq K\sqrt{A_\alpha}$?
% \end{problem}

\begin{problem}
    What is the function $f(A_\lambda,\lambda)$ so that for every frame $(\phi_i)_{i\in I}$ and saturation level $\lambda > 0$ such that the associated saturated measurement operator $\Sl$ is one-to-one on $\mathbb{B}_{\RR^n}$,  its optimal lower Lipschitz bound satisfies $\kappa_L\geq f(A_\lambda,\lambda)$?
\end{problem}

Notably, for frames in $\RR^n$ with $n+1$ elements, we can remove $\lambda$ from the minimum such that the lower Lipschitz bound becomes simply $\tfrac{1}{2}\sqrt{A_\lambda}$.
\begin{proposition}
    Let $(\phi_i)_{i=1}^{n+1}$ be a frame for $\RR^n$ and $\lambda>0$ such that $\Sl$ is one-to-one on $\mathbb{B}_{\RR^n}$. The optimal lower Lipschitz bound for $S_{\lambda}$ satisfies
    \begin{equation}\label{eq:satspecial}
        \tfrac{1}{2}\sqrt{A_\lambda} \leq \kappa_L \leq \sqrt{A_\lambda}.
    \end{equation}
    If additionally, $(\phi_i)_{i=1}^{n+1}$ is a finite unit norm tight frame for $\RR^n$ and $1>\lambda>0$ then the optimal lower Lipschitz bound for $S_{\lambda}$ satisfies
    \begin{equation}\label{eq:satspecialetf}
        \tfrac{1}{2\sqrt{n}} \leq \kappa_L \leq \tfrac{1}{\sqrt{n}}.
    \end{equation}
    % \begin{enumerate}[(i)]
    %     %\item The saturated measurement operator $S_{\lambda_c}$ with critical saturation level $\lambda_c$ is one-to-one.
    %     \item $A_{\lambda_c}=\tfrac{1}{n}$.
    %     \item The optimal lower Lipschitz bound for $S_{\lambda_c}$ satisfies
    %     \begin{equation}
    %         \tfrac{1}{\sqrt{2n}} \leq \kappa_L \leq \tfrac{1}{\sqrt{n}}.
    %     \end{equation}
    % \end{enumerate}
\end{proposition}
\begin{proof}
    First, recall that by Theorem \ref{thm:sat}, injectivity of $S_{\lambda}$ implies that for any $x\in \mathbb{B}_{\RR^n}$ the non-saturated coordinates are a frame and therefore $|\Il(x)|\geq n$. Since there are only $n+1$ frame elements, only up to one can be saturated. Let $x,y\in \RR^n$, we distinguish two cases. If $\I^\Delta(x,y)=\emptyset$ then by Theorem \ref{thm:sat}, we have that
    $$   \tfrac{1}{4}A_\lambda \|x-y\|^2\leq\|\Sl(x)-\Sl(y)\|^2.
    $$
  We now assume that $\I^\Delta(x,y)\neq\emptyset$. As $S_{\lambda}$ is injective and the frame has $n+1$ elements, the same single coordinate is saturated for both $x$ and $y$.  That is, $|\I^\Delta(x,y)|=1$ and $\Il(x)=\Il(y)=I\setminus \I^\Delta(x,y)$. We deduce
    \begin{align*}
        \|\Sl(x)-\Sl(y)\|^2
        =\sum_{i\in \I(x)}|\langle x-y,\phi_i\rangle|^2 + 4\lambda^2 \geq A_{\lambda} \Vert x-y \Vert^2.
    \end{align*}
Thus, in both cases we have the lower bound in \eqref{eq:satspecial}.

    To prove \eqref{eq:satspecialetf}, note that in the case that $(\phi_i)_{i=1}^{n+1}$ is a frame of unit vectors and $1>\lambda>0$, we have that $A_{\lambda}$ coincides with the smallest optimal lower frame bound of an $n$-element subset of $(\phi_i)_{i=1}^{n+1}$.  Now recall that the frame bound of a unit norm tight frame in $\RR^n$ with $n+1$ elements is given by $\tfrac{n+1}{n}$. Since any vector of a  unit norm tight frame is a tight frame for its own span with bound $1$ it follows that the optimal lower bound after removing an element becomes $\tfrac{n+1}{n}-1=\tfrac{1}{n}$. We deduce that $A_{\lambda}=\tfrac{1}{n}$, which together with the first statement finishes the proof.
    %Consequently, the value $A_{\lambda}(x)$ coincides with the optimal lower frame bound of the frame that arises when removing the corresponding element. By the nature of equiangular tight frames of this form, the optimal lower frame bound of the reminding frame is invariant under the choice of the removed element. Now recall that the frame bound of the original frame is given by $\tfrac{n+1}{n}$. Since any vector of a tight frame is a tight frame for its span with bound $1$ it follows that the optimal lower bound after removing an element becomes $\tfrac{n+1}{n}-1=\tfrac{1}{n}$.  We deduce that $A_{\lambda_c}=\tfrac{1}{n}$.
\end{proof}
Let us now turn to the open problem about the critical saturation level.
First, it is known that for any finite frame $(\phi_i)_{i\in I}$ for $\RR^n$ there is a critical value $\lambda_c>0$ so that $\Sl$ is one-to-one on the unit ball of $\RR^n$ if and only if $\lambda\geq \lambda_c$ \cite{alharbi2024sat}.  Furthermore, if $\lambda>\lambda_c$ then $\Sl$ is bi-Lipschitz stable.   Whether  saturation recovery is stable also at the critical saturation level was posed as an open problem in \cite{alharbi2024sat}.
A positive answer immediately follows from Theorem \ref{thm:sat} as the derived lower Lipschitz bound depends only on the fact that the associated saturated measurement operator $S_{\lambda_c}$ is one-to-one on the unit ball.
\begin{corollary}\label{cor:critsat}
    Saturation recovery is stable at the critical saturation level. In particular, if $(\phi_i)_{i\in I}$ is a finite frame of $\RR^n$ and $\lambda>0$ then $\Sl$ is injective on $\mathbb{B}_{\RR^n}$ if and only if $\Sl$ is bi-Lipschitz.
\end{corollary}

We note that for the gated measurement operator, the proof technique to derive a lower Lipschitz bound is not directly applicable since $\RR^n\setminus \mathbb{B}_{\RR^n}$ as the domain is not convex. By a modification, however, one might find a similar result. We leave this as an open problem for the future.

\subsection{Stability of phase retrieval - revisited}
Phase retrieval was the first of the three non-linear activation function problems to be considered in the context of frame theory~\cite{balan2006phase}. 
%has received a lot of attention in mathematical research so far~\cite{cahill2016infinitephase}. Hence,
Results on phase retrieval therefore have served as an important inspiration for approaching the injectivity and stability of ReLU layers and saturation recovery. Now, by considering Theorem \ref{thm:phaseinj} $(ii^*)$ we can turn around the direction of inspiration and revisit the stability of phase retrieval to slightly improve the known results from~\cite{bandeira2014savingphase}.\\

% We revisit the optimal lower Lipschitz stability results for (real) phase retrieval from~\cite{bandeira2014savingphase}, and contextualize them within the perspective of considering input-dependent sub-collections (as in Theorem \ref{thm:phaseinj} $(ii^*)$). While this marginally improves the known results, it nicely bridges to our main results on the optimal lower Lipschitz bounds for ReLU layers and saturation recovery, presented in Sections \ref{sec:relu} and \ref{sec:sat}.\\

To formulate their result on the optimal lower Lipschitz bound in phase retrieval, the authors of~\cite{bandeira2014savingphase} make use of the concept of the $\sigma$\textit{-strong complement property}. A frame $(\phi_i)_{i\in I}$ fulfills the  $\sigma$-strong complement property if for any $J\subseteq I$ the larger of the lower frame bounds of $(\phi_i)_{i\in J}$ and $(\phi_i)_{i\in J^c}$ is at least $\sigma^2$. As with characterizing injectivity, it turns out that it is sufficient to only consider the collections $(\phi_i)_{i\in I^+(x,y)}$ and $(\phi_i)_{i\in I^-(x,y)}$ for any $x,y\in \RR^n$ to obtain the result. Denoting by $A^+_{|\;\cdot\;|}(x,y)$ the optimal lower frame bound of $(\phi_i)_{i\in I^+(x,y)}$ and by $A^-_{|\;\cdot\;|}(x,y)$ the optimal lower frame bound of $(\phi_i)_{i\in I^-(x,y)}$ and defining
\begin{equation}\label{eq:Apr}
    A_{|\;\cdot\;|} = \min_{x,y\in \RR^n}\max\{A^+_{|\;\cdot\;|}(x,y),A^-_{|\;\cdot\;|}(x,y)\},
\end{equation}
the following was proven.
%The largest value $\sigma$ such that the $\sigma$-strong complement property holds for a given frame determines the bounds.
\begin{theorem}[Bandeira et al., 2014]\label{thm:Blip}
    Let $\kappa_L$ be the optimal lower Lipschitz bound of $\M$, and let $\sigma$ be the largest value such that the $\sigma$-strong complement property holds then
    \begin{equation}\label{E:sigma}
        \sigma \leq \kappa_L \leq \sqrt{2}\sigma.
    \end{equation}
    Moreover, 
    \begin{equation}\label{E:A}
        \sqrt{A_{|\;\cdot\;|}} \leq \kappa_L \leq \sqrt{2A_{|\;\cdot\;|}}.
    \end{equation}
\end{theorem}
%It is important to note that in phase retrieval the metric $d(x,y)=\min\{\Vert x-y \Vert, \Vert x+y \Vert\}$ on $\RR/_\sim$ is used.
Although not explicitly stated in the paper, the moreover part of Theorem \ref{thm:Blip} follows directly from the proof of Theorem 18 in~\cite{bandeira2014savingphase}.
Now, since $\sigma^2$ is obtained by taking a minimum over all partitions of $(\phi_i)_{i\in I}$ into two sets and $A_{|\;\cdot\;|}$ is obtained by taking a minimum only over partitions of a particular form, it must hold that $\sigma^2\leq A_{|\;\cdot\;|}$. Furthermore, it follows from Theorem \ref{thm:Blip} that $\sigma^2\leq A_{|\;\cdot\;|}\leq 2\sigma^2$. Indeed, there are cases where $\sigma^2= A_{|\;\cdot\;|}$ and cases where $A_{|\;\cdot\;|}= 2\sigma^2$.  One can check that if  $(\phi_i)_{i\in I}$ is the Mercedes-Benz frame in $\RR^2$ then $A_{|\;\cdot\;|}=\sigma^2$ and if $(\phi_i)_{i\in I}$ consists of two copies of a single non-zero vector in $\RR$ then $A_{|\;\cdot\;|}= 2\sigma^2$.  Thus, the following gives a slight improvement of the original statement in~\cite{bandeira2014savingphase}.
\begin{corollary}
    The optimal lower Lipschitz bound of $\M$ satisfies
    \begin{equation}
        \sqrt{A_{|\;\cdot\;|}} \leq \kappa_L \leq \sqrt{2}\sigma.
    \end{equation}
\end{corollary}

\section*{Funding}
The research of D. Freeman was supported by NSF grant 2154931. The research of D. Haider was supported by the Marietta-Blau Fellowship of the OeAD (MMC-2023-06983).

% In both, the ReLU and saturation case, there is an analog to the value in \eqref{eq:Apr} that can be used to identify an interval where the optimal lower Lipschitz bounds must lie.
%Note that the peculiarity in \eqref{eq:Apr} is again that we may always choose the better stability constant from two collections.
%Using this, we are able to derive explicit lower Lipschitz bounds for ReLU layers and saturation recovery.

% (WIP)\\
% We consider $x=u$ and $y=-u$. Analog to the estimation in the proof of Theorem \ref{thm:sat} we use the mean $\tfrac{x+y}{2}=0$ for the first step of the estimation. Noting that $\Il(0)=I$, we obtain
% \begin{align*}
%     \|&\Sl(x)-\Sl(y)\|^2 \geq \sum_{i\in (\Il(x)\cup \Il(y))} \tfrac{1}{4} \vert \langle 2 u, \phi_i\rangle \vert^2 + \sum_{i\in \Il^\Delta(x,y)} 4\lambda^2 \geq \tfrac{1}{4}A_\lambda \Vert 2 u \Vert^2 = \tfrac{1}{4}A_\lambda \Vert x-y \Vert^2.
% \end{align*} 
% Equality is attained if $\langle x,\phi_i\rangle=\lambda$ for all $i\in \Il(y)\setminus \Il(x)$ and $\langle y,\phi_i\rangle=\lambda$ for all $i\in \Il(x)\setminus \Il(y)$ and $\Il(x^*)\subseteq (\Il(x)\cup \Il(y))$ and $\Il^\Delta(x,y)=\emptyset$.\\
% On the other hand, if $\Il^\Delta(x,y)\neq\emptyset$, this sets the stage for the trivial lower bound that arises from using only one coordinate in $\Il^\Delta(x,y)$, i.e.,
% \begin{align*}
%     \|&\Sl(x)-\Sl(y)\|^2 \geq 4\lambda^2 = \lambda^2 \Vert x-y\Vert^2.
% \end{align*}
% Combined with the first approach, this shows the claim.

\bibliographystyle{abbrv}
\bibliography{references}

\end{document}